\documentclass[11pt,a4paper]{article}
\usepackage[utf8]{inputenc}
\usepackage[T1]{fontenc}
\usepackage{lmodern}
\usepackage[margin=1in]{geometry}
\usepackage{cite}
\usepackage{graphicx}
\usepackage{epstopdf}
\usepackage{physics}
\usepackage[caption=false,font=footnotesize]{subfig}
\usepackage{amsmath,amssymb,amsfonts,amsthm}
\usepackage{multirow}
\usepackage{makecell}
\usepackage{booktabs}
\usepackage{float}
\usepackage[normalem]{ulem}
\usepackage{xcolor}
\usepackage{dcolumn}
\usepackage{tabu}
\usepackage{bm}
\usepackage{hyperref}
\hypersetup{colorlinks=true,linkcolor=blue,citecolor=blue,urlcolor=cyan}

\def\1{\mbox{\boldmath$1$}}

\def\B{\mbox{\boldmath$B$}}
\def\C{\mbox{\boldmath$C$}}
\def\D{\mbox{\boldmath$D$}}

\def\H{\mbox{\boldmath$H$}}
\def\L{\mbox{\boldmath$L$}}
\def\I{\mbox{\boldmath$I$}}

\def\K{\mbox{\boldmath$K$}}
\def\M{\mbox{\boldmath$M$}}
\def\N{\mbox{\boldmath$N$}}

\def\P{\mbox{\boldmath$P$}}
\def\Q{\mbox{\boldmath$Q$}}
\def\R{\mbox{\boldmath$R$}}
\def\S{\mbox{\boldmath$S$}}

\def\W{\mbox{\boldmath$W$}}
\def\X{\mbox{\boldmath$X$}}
\def\Y{\mbox{\boldmath$Y$}}

\newcommand{\LambdaB}{\bm{\Lambda}}

\def\b{\mbox{\boldmath$b$}}

\def\f{\mbox{\boldmath$f$}}
\def\g{\mbox{\boldmath$g$}}

\def\q{\mbox{\boldmath$q$}}

\def\t{\mbox{\boldmath$t$}}
\def\u{\mbox{\boldmath$u$}}
\def\v{\mbox{\boldmath$v$}}

\def\x{\mbox{\boldmath$x$}}
\def\y{\mbox{\boldmath$y$}}
\def\z{\mbox{\boldmath$z$}}
\def\u{\mbox{\boldmath$u$}}

\def\0{\mbox{\boldmath$0$}}

\theoremstyle{plain}
\newtheorem{theorem}{Theorem}[section]
\newtheorem{proposition}[theorem]{Proposition}

\theoremstyle{definition}
\newtheorem{definition}[theorem]{Definition}
\newtheorem{assumption}[theorem]{Assumption}
\theoremstyle{remark}
\newtheorem{remark}[theorem]{Remark}
\newcommand*\proofdeltaVtitle{PROOF OF EQ.~\eqref{eq:deltaV}}
\newcommand*\proofSlopeTitle{PROOF OF EQ.~\eqref{eq:slope0}}

\begin{document}
\title{Learning Dissipative Dynamics with Dissipativity-by-Construction Discrete-Time Neural Networks}
\author{Tuan Luong \and Hyungpil Moon}
\date{}
\maketitle
\begin{center}
\small Department of Mechanical Engineering, Sungkyunkwan University,\\
Suwon 16419, South Korea\\[0.3em]
\texttt{luongtuan@g.skku.edu} \quad \texttt{hyungpil@g.skku.edu}\\[0.3em]
Corresponding author: Hyungpil Moon (\texttt{hyungpil@g.skku.edu})
\end{center}
\begin{center}
\footnotesize This material is based upon work supported by the Air Force Office of Scientific Research under award number FA2386-24-1-4039. This work was supported by the National Research Foundation of Korea (NRF) grant funded by the Korea government (MSIT) (RS-2025-24683341).
\end{center}

\begin{abstract}
  Dissipativity is a fundamental system-theoretic property closely related to stability, passivity, and input--output stability, and is particularly important in robotics, where learned dynamics models are often embedded within feedback control loops. However, most existing approaches for learning dissipative dynamics are based on continuous-time formulations, which require ODE solvers during training or inference and can therefore be computationally expensive. Moreover, because practical implementations are inherently discrete-time, direct discretization of a continuous-time passive system does not necessarily preserve passivity, motivating the need for explicit discrete-time guarantees. This study proposes a method for learning incrementally dissipative dynamics from input--output time-series data using a deep multilayer perceptron formulated directly in discrete time. Through a constrained parameterization and a dedicated training procedure, the proposed model guarantees incremental dissipativity by construction rather than through regularization. Lyapunov-based analysis establishes the corresponding dissipativity and stability guarantees, while simulations on robotic dynamical systems demonstrate competitive prediction accuracy, computational efficiency, and consistent preservation of incremental dissipativity compared with baseline methods.
\end{abstract}


\section{Introduction}
Learning-based models of dynamical systems are increasingly used in simulation, prediction, and control pipelines. Compared with first-principles modeling, data-driven approaches can achieve high accuracy even when physical parameters are uncertain, unmodeled dynamics exist, or analytical modeling is intractable~\cite{lee2017gp}. However, despite strong short-horizon performance, purely data-driven models can suffer from distribution shift, compounding errors over long rollouts, and instability in feedback loops. This is mainly because standard neural architectures are optimized for pointwise prediction accuracy without explicitly enforcing structural properties of physical systems. As a result, even when training trajectories are stable, the learned dynamics may produce unstable or physically implausible behavior outside the training distribution~\cite{okamoto2025learning}. Therefore, ensuring stability and reliability is a key design objective for learned dynamics models intended for long-horizon simulation, control, and generalization.

To address the problem, recent research has explored ways to incorporate stability principles directly into learning frameworks. One direction leverages Lyapunov-based ideas to constrain or certify stability of neural systems, aiming to connect learning objectives with provable dynamical guarantees~\cite{chang2019neural}. Another direction focuses on continuous-time modeling through neural ordinary differential equations, providing a flexible representation that aligns naturally with physics-based time evolution~\cite{chen2018neural}. In parallel, theoretical approaches such as contraction-based learning of stable dynamical systems have been proposed to encourage globally consistent behavior~\cite{blocher2017learning}. These approaches highlight that enforcing structure can be essential for producing models that behave predictably beyond the training regime.

We address the problem of integrating incremental dissipativity into a neural network architecture for learning dynamical systems. Dissipativity~\cite{ren2020exponentially, koelewijn2021incremental} is important because it provides an energy-based framework for analyzing stability, robustness, and input--output consistency of dynamical systems. In particular, incremental dissipativity constrains the evolution of differences between trajectories through an incremental storage function and supply rate, preventing deviations between two input--output trajectories from amplifying beyond the energy injected by their input differences~\cite{lozano2013dissipative}. This property is especially useful for closed-loop and interconnected systems, since dissipative systems preserve desirable stability properties under interconnection; for example, parallel and negative-feedback interconnections of passive systems remain dissipative~\cite{bao2007process}. 
Although learning dissipative models has been extensively studied, most existing approaches focus on continuous-time dynamics. For examples, some approaches rely on specific port-Hamiltonian parameterizations~\cite{drgovna2022dissipative}, consider incremental dissipativity with multi-stage learning procedures~\cite{xu2023learning}, or use continuous-time Neural ODE projections to enforce dissipativity~\cite{okamoto2025learning}, which can introduce high computational cost and may affect modeling accuracy.

In this work, we pursue a different viewpoint: instead of treating dissipativity as an auxiliary penalty or post-hoc certificate, we design the hypothesis class itself so that the learned discrete-time dynamics are incrementally dissipative by construction. 
Specifically, we develop a deep discrete-time neural network architecture whose parameters are explicitly constrained to satisfy incremental dissipativity throughout training. 
This provides two main advantages: first, the learned model inherently preserves the desired dissipativity property without requiring projection onto a dissipativity-preserving space; second, the discrete-time formulation avoids the high computational cost associated with Neural ODE-based methods such as~\cite{okamoto2025learning}. 
Our computational complexity analysis and simulation results show that the proposed method achieves competitive prediction accuracy while being significantly more computationally efficient than continuous-time Neural ODE approaches. 
Moreover, since digital devices operate in discrete time, the proposed framework is naturally suited for practical deployment on digital hardware and embedded control platforms.
 
The remaining of this paper is organized as follows. 
Section~\ref{sec:nnmodel} presents the incrementally dissipative neural network model and its learning algorithm. 
Case studies demonstrating the performance of the proposed network and evaluations are presented in Section~\ref{sec:casestudy} and \ref{sec:evaluation}. 
We conclude the paper in Section~\ref{sec:conclusion}.
\section{Method} \label{sec:nnmodel}
\subsection{Problem formulation}
In this work, we consider discrete-time nonlinear systems of the form
\begin{equation}\label{eq:ss1}
    \x_{k+1} = \f(\x_{k}, \u_{k})
\end{equation}
\begin{equation}\label{eq:ss2}
    \y_{k} = \g(\x_{k}, \u_{k})
\end{equation}
where $\x_{k} \in \mathbb{R}^{n \times 1}$, $\u_{k} \in \mathbb{R}^{m \times 1}$, and $\y_{k} \in \mathbb{R}^{p \times 1}$ are the state, input, and output at time step $k$, respectively. We assume $\f(0)=0$ and $\g(0)=0$.
\begin{definition}
The system in Eqs.~\eqref{eq:ss1}--\eqref{eq:ss2} is said to be 
$(\Q,\S,\R)$-dissipative, where
$\Q \preceq 0$, $\Q \in \mathbb{R}^{p\times p}$,
$\S \in \mathbb{R}^{m\times p}$, and
$\R=\R^T \in \mathbb{R}^{m\times m}$,
if there exists a nonnegative storage function
$V_k(\Delta\x_k)\geq 0$ such that
\begin{equation}
\label{eq:qsrCond}
\begin{aligned}
&V_{k+1}(\Delta\x_{k+1})-V_k(\Delta\x_k)
\\
&\quad\leq
\begin{bmatrix}
\Delta\y_k\\
\Delta\u_k
\end{bmatrix}^{T}
\begin{bmatrix}
\Q & \S^T\\
\S & \R
\end{bmatrix}
\begin{bmatrix}
\Delta\y_k\\
\Delta\u_k
\end{bmatrix}.
\end{aligned}
\end{equation}
\end{definition}

Some important special cases of $(\Q, \S, \R)$ dissipativity include passivity and finite-gain $\mathcal{L}_2$ stability. In particular, a passive system corresponds to $\Q=\0$, $\R=\0$, and $\S=\frac{1}{2}\I$. An $\mathcal{L}_2$ system is dissipative with $\S=\0$ and $\Q=\frac{1}{\xi}\I$, where $\xi$ is the $\mathcal{L}_2$ gain of the system~\cite{brogliato2007dissipative}.

When the input--output data $(\u_k,\y_k)$ are given and modeled using a neural network, the system is represented as
\begin{equation}\label{eq:rnn}
    \x_{k+1} = f_{nn}(\x_{k},\u_k, \theta_{nn})
\end{equation}
\begin{equation}\label{eq:output}
    \hat{\y}_k = \C \x_{k} + \D \u_k + \b_y
\end{equation}
and the parameters are trained by solving
\begin{equation}\label{eq:nn1}
    \min_{\substack{\theta_{nn}, \C, \D}} \quad \sum_{i=1}^{N} \mathcal{L}(\hat{\y}_i, \y_i),
\end{equation}
where $\theta_{nn}$ are trainable weights of the network, $\C \in \mathbb{R}^{p \times n}$, $\D \in \mathbb{R}^{p \times m}$, $\b_y \in \mathbb{R}^{p \times 1}$, and $\hat{\y}_i, \y_i$ are the predicted output and the ground truth at time step $i$, respectively. It should be noted here that since state dynamics is nonlinear, the input-output relationship in Eqs.~(\ref{eq:rnn})-(\ref{eq:output}) is nonlinear. Indeed, the same discrete-time system structure has been successfully used to represent various types of nonlinear dynamics such as in~\cite{d2023incremental}. Our contribution goes beyond input--output modeling by preserving the dissipative properties of incrementally dissipative systems through the network parameter construction. 
In Section~\ref{sec:lmi_cond}, we derive a condition under which the discrete-time MLP system in Eqs.~(\ref{eq:rnn})--(\ref{eq:output}) is $\Q,\S,\R$-dissipative. 
Section~\ref{sec:model} then presents a training algorithm that enforces this dissipativity condition by construction, without solving iterative LMI problems.

\vspace{0.25em}
\noindent
\subsection{Necessary condition of a discrete-time dissipative neural network} \label{sec:lmi_cond}
Expressing the state dynamics using a discrete-time multilayer perceptron neural network, Eq.~\eqref{eq:rnn} can be expanded as follows
\begin{equation}\label{eq:newnn1}
\begin{aligned}
\z^{1}_{k} &= \phi(\W^{1}_{1} \u_k + \W^{2}_{1} \x_{k}+\b_{1}) \\
\z^2_{k} &= \phi(\W^2_{2} \z^1_{k}+\b_{2}) \\
&\vdots \\
\z^{L}_{k} &= \phi(\W^{L}_{2} \z^{L-1}_{k}+\b_{L}),
\end{aligned}
\end{equation}
\begin{equation}
    \x_{k+1} = \B \z^{L}_{k} + \b_x,
\end{equation}
\begin{equation}
    \hat{\y}_k = \C \x_{k} + \D \u_k + \b_y,
\end{equation}
where $\W^{1}_{1} \in \mathbb{R}^{q \times m}$, $\W^{2}_{1} \in \mathbb{R}^{q \times n}$, $\b_{i} \in \mathbb{R}^{q \times 1}$ $(i=1,\cdots,L)$, $\W^{i}_{2} \in \mathbb{R}^{q \times q}$ $(i=2,\cdots,L)$, $\B \in \mathbb{R}^{n \times q}$, and $\b_{x} \in \mathbb{R}^{n \times 1}$. The activation function is denoted by $\phi(\cdot)$.

Eq.~\eqref{eq:newnn1} can be represented in a stacked form as
\begin{equation}\label{eq:ss1_3}
    \x_{k+1} = \B \B_0 \phi(\W_{u} \u_k + \W_{1} \x_{k} + \W_{2} \z_{k}+\b_{z}) + \b_x
\end{equation}
\begin{equation}\label{eq:ss2_3}
    \hat{\y}_k = \C \x_{k} + \D \u_k + \b_y
\end{equation}
where
\begin{equation}
\label{eq:B0_Wu}
\begin{aligned}
\B_0 &=
\begin{bmatrix}
\0 & \0 & \cdots & \I^{q\times q}
\end{bmatrix}
\in \mathbb{R}^{q\times Lq},
\\
\W_u &=
\begin{bmatrix}
\W^{1}_{1} \\
\0 \\
\vdots \\
\0
\end{bmatrix}
\in \mathbb{R}^{Lq\times m}.
\end{aligned}
\end{equation}

\begin{equation}
\label{eq:W1_W2}
\begin{aligned}
\W_1 &=
\begin{bmatrix}
\W^{2}_{1} \\
\0 \\
\vdots \\
\0
\end{bmatrix}
\in \mathbb{R}^{Lq\times n},
\\[1mm]
\W_2 &=
\begin{bmatrix}
\0 & \0 & \cdots & \0 \\
\W^{2}_{2} & \0 & \cdots & \0 \\
\vdots & \ddots & \ddots & \vdots \\
\0 & \cdots & \W^{L}_{2} & \0
\end{bmatrix}
\in \mathbb{R}^{Lq\times Lq}.
\end{aligned}
\end{equation}
\begin{equation}
\begin{aligned}
\b_z &=
\begin{bmatrix}
\b_1^T & \b_2^T & \cdots & \b_L^T
\end{bmatrix}^T
\in \mathbb{R}^{Lq\times 1},
\\
\z_k &=
\begin{bmatrix}
(\z_k^1)^T & (\z_k^2)^T & \cdots & (\z_k^L)^T
\end{bmatrix}^T
\in \mathbb{R}^{Lq\times 1}.
\end{aligned}
\end{equation}
\begin{assumption}\label{assump1}
The activation function $\phi$ is piecewise differentiable and slope-restricted as
\begin{equation}
0 \leq \alpha \leq \frac{\phi(x_2)-\phi(x_1)}{x_2-x_1} \leq \beta.
\end{equation}
\end{assumption}
This assumption holds for widely used activation functions such as ReLU, Leaky ReLU ($\alpha > 0$ for stable implementation), sigmoid, tanh, softplus, ELU, SELU, Hard Sigmoid, hard tanh. While the dissipativity condition can be extended to general finite values of $\alpha$ and $\beta$, the proposed training algorithm focuses on the case $\alpha \geq 0$. One possible solution for the case $\alpha<0$ is to reformulate the problem using the shifted activation function $\tilde{\phi}(\v)=\phi(\v)-\alpha \v$. This extension will be investigated in our future work.

\begin{proposition}
The systems described in \ref{eq:ss1_3}--\ref{eq:ss2_3} are $(\Q, \S, \R)$ dissipative if there exists a positive definite matrix $\P = \P^T \succ 0$ such that
\begin{equation}
\label{eq:lmi}
\begin{aligned}
&\left[
\begin{array}{ccc}
\P+\C^T\Q\C
& -\gamma\W_{1}^T\LambdaB
& \C^T\S^T+\C^T\Q\D
\\
-\gamma\LambdaB\W_{1}
& \LambdaB^*
& -\gamma\LambdaB\W_u
\\
\S\C+\D^T\Q\C
& -\gamma\W_u^T\LambdaB
& \R_1
\end{array}
\right]
\\[-1mm]
&\quad
+\rho
\begin{bmatrix}
\W_1^T\\
\W_2^T\\
\W_u^T
\end{bmatrix}
\LambdaB
\begin{bmatrix}
\W_1 & \W_2 & \W_u
\end{bmatrix}
\succ 0 .
\end{aligned}
\end{equation}
where $\LambdaB^* = \LambdaB - \gamma \W_{2}^* - \B_0^T\B^T \P \B\B_0$, $\W^*_{2} = \LambdaB \W_{2} + \W_{2}^T \LambdaB$, $\R_1 = \R + \S \D + \D^T \S^T + \D^T \Q \D,$ and $\LambdaB = \mathrm{diag}(\lambda_1,\lambda_2,\cdots,\lambda_{Lq})$ with $\lambda_i>0$ $(i=1,\cdots,Lq)$.
\end{proposition}

\begin{proof}
Consider the storage function
\begin{equation}\label{eq:Vkz}
    V_k(\Delta \x_k) = \Delta \x_k^T \P \Delta \x_k.
\end{equation}
Left and right multiplying \eqref{eq:lmi} by $[\Delta \x_k^T \quad \Delta \z_k^T \quad \Delta \u_k^T]$ and its transpose yields (see Appendix~\ref{sec:provedeltaV})
\begin{equation}
\label{eq:deltaV}
\begin{aligned}
&V_{k+1}(\Delta\x_{k+1})
- V_k(\Delta\x_k)
\\
&\leq
\begin{bmatrix}
\Delta\y_k\\
\Delta\u_k
\end{bmatrix}^{T}
\begin{bmatrix}
\Q & \S^{T}\\
\S & \R
\end{bmatrix}
\begin{bmatrix}
\Delta\y_k\\
\Delta\u_k
\end{bmatrix}
\\
&\quad+
\begin{bmatrix}
\Delta\v_k\\
\Delta\z_k
\end{bmatrix}^{T}
\begin{bmatrix}
\rho\LambdaB & -\gamma\LambdaB\\
-\gamma\LambdaB & \LambdaB
\end{bmatrix}
\begin{bmatrix}
\Delta\v_k\\
\Delta\z_k
\end{bmatrix}
\\
&\leq
\begin{bmatrix}
\Delta\y_k\\
\Delta\u_k
\end{bmatrix}^{T}
\begin{bmatrix}
\Q & \S^{T}\\
\S & \R
\end{bmatrix}
\begin{bmatrix}
\Delta\y_k\\
\Delta\u_k
\end{bmatrix}.
\end{aligned}
\end{equation}
where $\Delta \v_k = \Delta \W_u \u_k + \Delta \W_1 \x_k + \Delta \W_2 \z_k$. Moreover, using the slope restriction (see Appendix~\ref{sec:slopeproof}),
\begin{multline}\label{eq:slope0}
\begin{bmatrix}
\v_k\\
\z_k
\end{bmatrix}^T
\begin{bmatrix}
\rho\LambdaB & -\gamma\LambdaB\\
-\gamma\LambdaB & \LambdaB
\end{bmatrix}
\begin{bmatrix}
\v_k\\
\z_k
\end{bmatrix}
\le 0
\end{multline}
Combining \ref{eq:deltaV} and \ref{eq:slope0} yields the dissipativity condition \ref{eq:qsrCond}.
\end{proof}

\begin{remark}
Since the second term in Eq.~(\ref{eq:lmi}) is positive semidefinite for $\rho \ge 0$ and $\LambdaB \succ 0$, a sufficient (more conservative) condition for \ref{eq:lmi} is
\begin{equation}\label{eq:lmi2}
\begin{aligned}
\begin{bmatrix}
\P + \C^T \Q \C
& -\gamma \W_{1}^T \LambdaB
& \C^T \S^T + \C^T \Q \D \\

-\gamma \LambdaB \W_{1}
& \LambdaB^*
& -\gamma \LambdaB \W_u \\

\S \C + \D^T\Q\C
& -\gamma \W_u^T \LambdaB
& \R_1
\end{bmatrix}
\succ 0.
\end{aligned}
\end{equation}
\end{remark}

\subsection{Training algorithm for the dissipative network}
\label{sec:model}

The following procedure finds network parameters that satisfy the constraint in \eqref{eq:lmi}.

\noindent
\textbf{Step 1:} Choose $\D$ such that
\begin{equation}
\R_1 = \R + \S \D + \D^T \S^T + \D^T \Q \D \succ 0
\quad
\text{(see Appendix~\ref{sec:D}).}
\end{equation}

\noindent
If $\R_1\succ 0$, the LMI in \eqref{eq:lmi} can be rewritten using the Schur complement as
\begin{equation}
\label{eq:lmi1}
\begin{aligned}
&
\begin{bmatrix}
\P+\C^T\Q\C
& -\gamma\W_1^T\LambdaB
\\
-\gamma\LambdaB\W_1
& \LambdaB-\gamma\W_2^*-\P
\end{bmatrix}
\\[1mm]
&\quad -
\begin{bmatrix}
\C^T(\S^T+\Q\D)
\\
-\gamma\LambdaB\W_u
\end{bmatrix}
\R_1^{-1}
\begin{bmatrix}
\C^T(\S^T+\Q\D)
\\
-\gamma\LambdaB\W_u
\end{bmatrix}^{T}
\succ 0 .
\end{aligned}
\end{equation}

\noindent
\textbf{Step 2:} Choose $\C$, $\B$, and $\LambdaB\W_u$ as free variables, and find $\P$, $\LambdaB$ such that
\begin{multline}\label{eq:lmi2_train}
\begin{bmatrix}
\P + \C^T \Q \C & -\gamma\W_1^T \LambdaB \\
-\gamma\LambdaB  \W_1 & \LambdaB -\gamma \W^*_{2} -\P
\end{bmatrix}
\succ \X^T \X \\
+\begin{bmatrix}
\C^T \S^T + \C^T \Q \D \\
-\gamma \LambdaB  \W_u
\end{bmatrix}
\R_1^{-1}
[\S \C + \D^T \Q \C \quad  -\gamma \W_u^T \LambdaB ]
\end{multline}
Denote
\begin{multline}
\H =
\begin{bmatrix}
\H_{11} & \H_{12} \\
\H_{21} & \H_{22}
\end{bmatrix}
= \X \X^T \\
+ \begin{bmatrix}
\C^T \S^T + \D^T \Q \C \\
-\gamma\LambdaB \W_u
\end{bmatrix}
\R_1^{-1}
[\S \C + \C^T \Q \D \quad  -\gamma \W_u^T\LambdaB]
\end{multline}

We construct $\H^*$ satisfying
\begin{multline}\label{eq:lmi0}
\H^{*}=
\begin{bmatrix}
\H_{11} + \epsilon_1\I & \H_{12}^* + \Y\\
\H_{12}^{*T} + \Y^T &  \H_{22}+ \epsilon_1\I
\end{bmatrix}
\succ
\begin{bmatrix}
\H_{11} & \H_{12}\\
\H_{12}^T &  \H_{22}
\end{bmatrix},
\end{multline}
where $\H_{12}^*$ is obtained from $\H_{12}$ by zeroing out all non-diagonal block matrices while preserving the diagonal blocks in the same form as $\W_1$ in \eqref{eq:W1_W2}. The matrix $\Y \in \mathbb{R}^{Lq \times Lq}$ is trainable and has the same structure as $\W_1$. A sufficient condition for \eqref{eq:lmi0} is
\begin{equation}
\epsilon_1 \geq \norm{\H_{12}^* + \Y-\H_{12}}_2.
\end{equation}

Finally, $\P$ and $\gamma \W_1^T\LambdaB$ are chosen as
\begin{equation}\label{eq:P}
    \P = \H_{11}-\C^T \Q \C + \epsilon_1\I \succ 0,
\end{equation}
\begin{equation}\label{eq:Wz_cond2}
    -\gamma \W_1^T\LambdaB  = \H_{12}^* + \Y.
\end{equation}
By choosing $\B$ freely and selecting $\LambdaB$ as
\begin{equation}\label{eq:lambda_cond2}
\begin{aligned}
\LambdaB =
\Bigg(
&\max_i \Bigg|
eig_i\Big(
\H_{22}
+\B_0^T\B^T\P\B\B_0\\
&\qquad
+\gamma\W^*_{2}
+\epsilon_1\I
\Big)
\Bigg|
+\epsilon_2
\Bigg)\I.
\end{aligned}
\end{equation}
with $\epsilon_2 > 0$, the constraint \eqref{eq:lmi1} is satisfied.

\subsection{Improved training algorithm} 
It is observed that the matrices $\W_2^{i}$ for $i=2,\ldots,L$ significantly affect the performance of deep neural networks as the depth $L$ increases. In particular, if $\|\W_2^{i}\|$ becomes too small or too large, it may lead to vanishing or exploding gradients, respectively~\cite{zhang2021dive}, due to repeated multiplication by the same weight matrix during backpropagation. 

While exploding gradients can be mitigated by gradient clipping~\cite{goodfellow2016deep}, vanishing gradients are more challenging to address in our setting because layer normalization cannot be applied under our constraints. Motivated by identity connections, which have been shown to be effective in alleviating vanishing gradients in ResNets~\cite{he2016deep} and through gating mechanisms in LSTMs~\cite{hochreiter1997long}, we modify the network by replacing $\W_2^{i}$ for $i=2,\ldots,L$ with an identity connection. Specifically, the network is modified to
$\W_2^{i} + \I_q$ for $i=2 \cdots L$ where $\I_q$ is the identity matrix with size $q \times q$. 
\begin{equation}\label{eq:lambda_cond3}
\begin{aligned}
\LambdaB =
\Bigg(
&\max_i \Bigg|
eig_i \Bigg(
\Big[
\I-\gamma(\I_{w2}+\I_{w2}^{T})
\Big]^{-1}\\
&\quad \times
\Big[
\H_{22}
+\B_0^T\B^T\P\B\B_0
+\gamma\W^*_{2}
+\epsilon_1\I
\Big]
\Bigg)
\Bigg|
+\epsilon_2
\Bigg)\I.
\end{aligned}
\end{equation}
where $\epsilon_2 > 0,$ $\I_{Lq}$ is an identity matrix size $Lq \times Lq$, and
\begin{equation} \label{eq:Iw2}
\I_{w2} =    
\begin{bmatrix}
\0 & \0 & \cdots & \0 \\
\I_{q} & \0 & \cdots & \0 \\
\vdots & \vdots & \ddots & \vdots \\
\0 & \0 & \I_{q} & \0 \\
\end{bmatrix} 
\in \mathbb{R}^{L q \times L q}
\end{equation}
\subsection{Computational complexity of the algorithm}
The computational complexity per iteration of the training algorithm can be estimated as follows. The computational complexity of the proposed algorithm can be estimated as follows. 
The dominant computations in Eqs.~(21)--(29) involve large matrices of size $(Lq)\times(Lq)$. 
In particular, Eq.~(25) requires matrix multiplications between $(Lq)\times(Lq)$ matrices, resulting in a complexity of $\mathcal{O}\big((Lq)^3\big)$. 
Similarly, the norm computation in Eq.~(26) for a matrix of size $(Lq)\times(Lq)$ also incurs a complexity of $\mathcal{O}\big((Lq)^3\big)$. 
In addition, Eq.~(31) requires an eigenvalue decomposition of an $(Lq)\times(Lq)$ matrix, whose complexity is $\mathcal{O}\big((Lq)^3\big)$. 
Thus, the per-iteration complexity of the proposed algorithm is dominated by cubic-order matrix operations, yielding $\mathcal{O}\big((Lq)^3\big)$. 
In contrast, Neural ODE-based continuous-time approaches~\cite{okamoto2025learning} depend on both the ODE solver and network structure, with rough complexity $N_{\mathrm{eval}}\mathcal{O}\big((Lq)^2\big)$, where $N_{\mathrm{eval}}$ is the number of solver function evaluations. 
Since $N_{\mathrm{eval}}$ can reach hundreds or thousands of effective layers~\cite{finlay2020train, chen2018neural}, it often becomes a major computational bottleneck.

\section{Case studies} \label{sec:casestudy}
In this section, the proposed dissipative network will be evaluated through the following robotic modeling examples: a mass-spring-damper model and n-DOF manipulaltor dynamics, which were both proven to be incrementally dissipative. Note that the following storage functions are used only to establish the dissipativity properties of the systems and are not required as prior information by our algorithm. The matrix $\P$ associated with the storage function of the proposed dissipative neural network is learned during training as shown in Eq.(~\ref{eq:P}).
\subsection{Mass-spring-damper system}
The dynamics of a mass--spring--damper (MSD) system are given by
\begin{equation}
    m\ddot{x}(t) + b\dot{x}(t) + kx(t) = u(t),
\end{equation}
where $m$, $b$, and $k$ denote the mass, damping, and spring coefficients, respectively.
The input is $f(t)$ and the position output is $x(t)$.

From the energy-based analysis (see Appendix~\ref{appendix:msd}), the mass--spring--damper (MSD) system is
\emph{incrementally dissipative}. For two trajectories $(x_1,\dot{x}_1,u_1)$ and $(x_2,\dot{x}_2,u_2)$, define
$\Delta x := x_1-x_2,\quad \Delta \dot{x} := \dot{x}_1-\dot{x}_2,\quad \Delta u := u_1-u_2.$
The incremental storage function is chosen as
\begin{equation}\label{eq:inc_storage_msd}
V_\Delta(\Delta x,\Delta \dot{x})
= \frac{1}{2}k(\Delta x)^2 + \frac{1}{2}m(\Delta \dot{x})^2,
\end{equation}
and the incremental supply rate is defined by
\begin{equation}\label{eq:inc_supply_msd}
w_\Delta(\Delta u,\Delta \dot{x})
= \Delta \dot{x}\,\Delta u
=
\begin{bmatrix}\Delta \dot{x} \\ \Delta f\end{bmatrix}^{T}
\begin{bmatrix}0 & \frac{1}{2}\\ \frac{1}{2} & 0\end{bmatrix}
\begin{bmatrix}\Delta \dot{x} \\ \Delta f\end{bmatrix}.
\end{equation}
\begin{figure}[!t]
\centering

\begin{minipage}[b]{0.38\linewidth}
    \centering
    \subfloat[Noise-free training.]{
        \includegraphics[width=\linewidth]{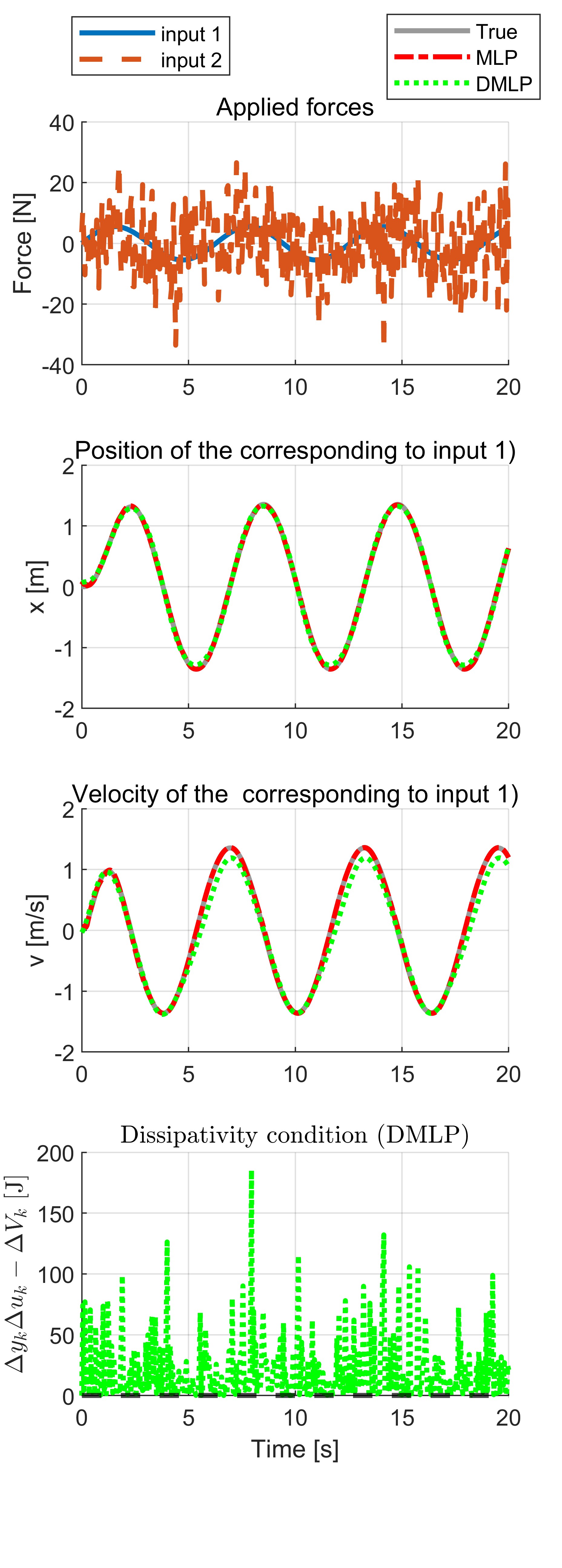}
        \label{fig:compare_msd_incremental}
    }
\end{minipage}
\hfill
\begin{minipage}[b]{0.42\linewidth}
    \centering
    \subfloat[Noisy training.]{
        \includegraphics[width=0.85\linewidth]{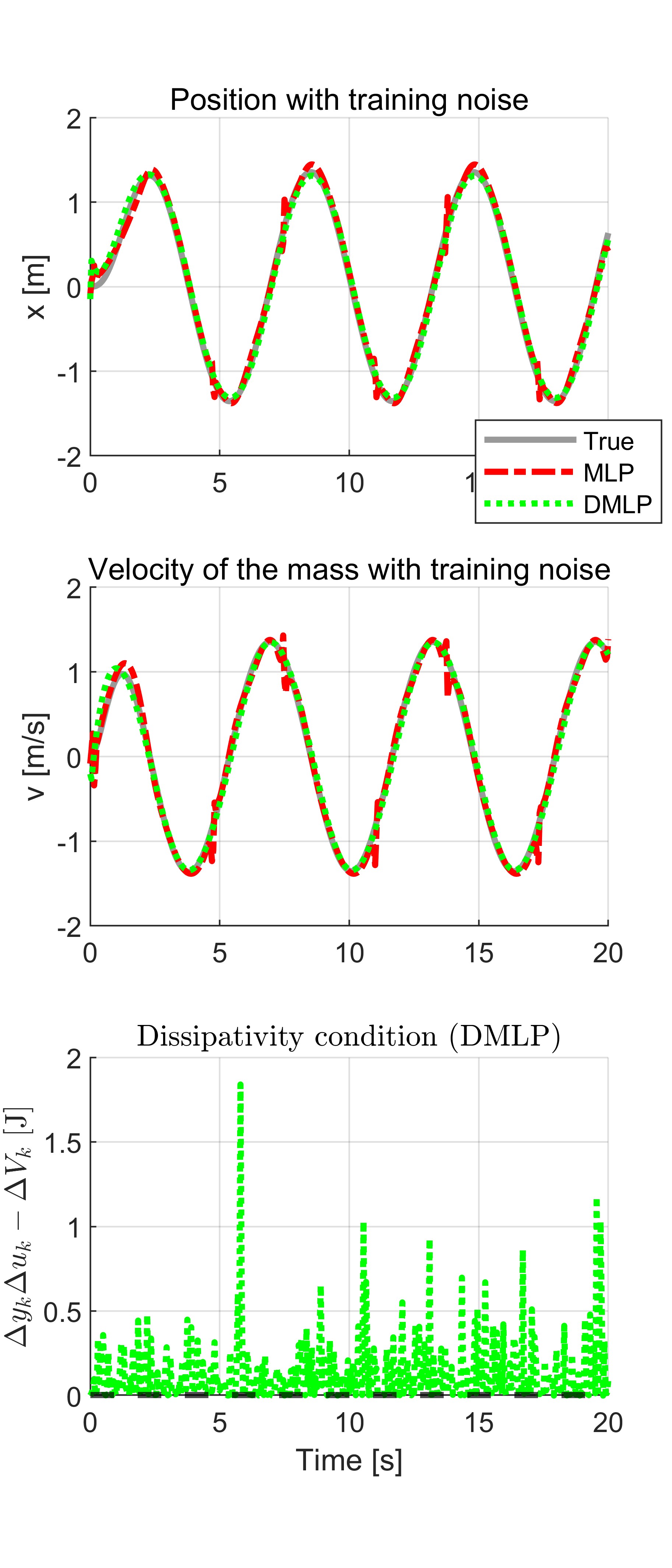}
        \label{fig:compare_msd_incremental_noise}
    }

    \vspace{0.6em}

    \subfloat[Dissipativity condition of the MLP model]{
        \includegraphics[width=0.82\linewidth]{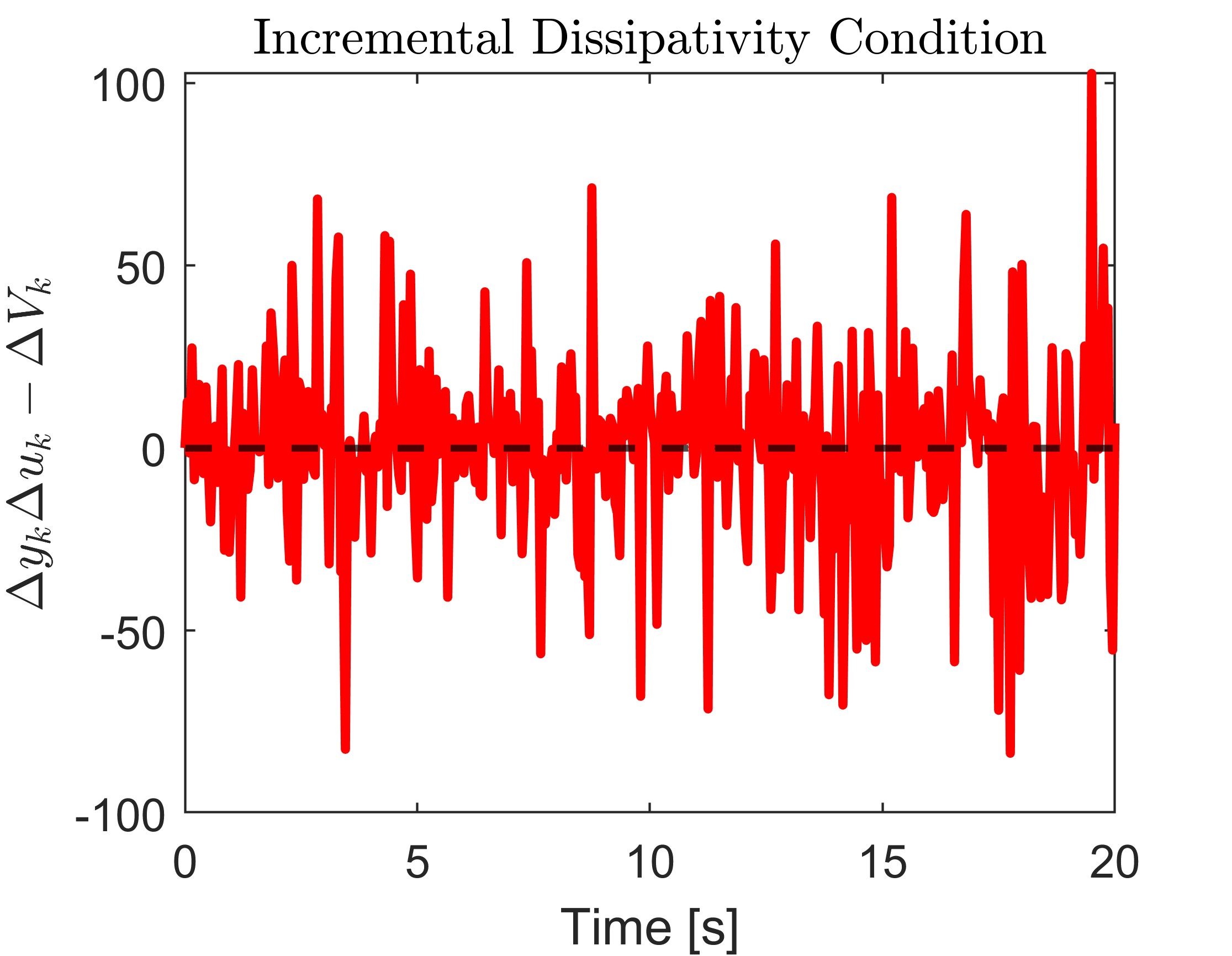}
        \label{fig:mlp_diss_check}
    }
\end{minipage}

\caption{Performance comparison for the mass--spring--damper (MSD) system. 
(a) Noise-free case, where data from $0$--$10\,\mathrm{s}$ are used for training and data from $10$--$20\,\mathrm{s}$ for testing; an additional disturbed input is used to verify incremental dissipativity. 
(b) Noisy case with Gaussian training noise $(\sigma=0.3)$, showing that the proposed dissipative network maintains good prediction accuracy and robustness. 
(c) The standard MLP without dissipativity constraints fails to preserve incremental dissipativity, evaluated using Eq.~(\ref{eq:qsrCond}) with known MSD parameters.}
\label{fig:compare_msd_all}
\end{figure}
\subsection{n-DOF manipulator dynamics}

The dynamics of an $n$-DOF manipulator system are given by
\begin{equation}
    \M \ddot{\q}(t) + \D \dot{\q}(t) + \K \q(t) = \tau(t),
\end{equation}
where $\q(t)\in\mathbb{R}^{n}$ is the joint position vector, $\tau(t)\in\mathbb{R}^{n}$ is the input torque vector, and $\M$, $\D$, and $\K$ denote the inertia, damping, and stiffness matrices, respectively. 
The input is $\tau(t)$ and the output is chosen as the joint velocity $\dot{\q}(t)$. From the energy-based analysis (see Appendix~\ref{appendix:ndof_manipulator}), the $n$-DOF manipulator system is
\emph{incrementally dissipative}. For two trajectories 
$(\q_1,\dot{\q}_1,\tau_1)$ and $(\q_2,\dot{\q}_2,\tau_2)$, define
\[
\Delta \q := \q_1-\q_2,\quad
\Delta \dot{\q} := \dot{\q}_1-\dot{\q}_2,\quad
\Delta \tau := \tau_1-\tau_2 .
\]
The incremental storage function is chosen as
\begin{equation}\label{eq:inc_storage_ndof}
V_\Delta(\Delta \q,\Delta \dot{\q})
=
\frac{1}{2}\Delta \dot{\q}^{T}\M \Delta \dot{\q}
+
\frac{1}{2}\Delta \q^{T}\K \Delta \q,
\end{equation}
and the incremental supply rate is defined by
\begin{equation}\label{eq:inc_supply_ndof}
w_\Delta(\Delta \tau,\Delta \dot{\q})
=
\Delta \dot{\q}^{T}\Delta \tau
=
\begin{bmatrix}
\Delta \dot{\q} \\
\Delta \tau
\end{bmatrix}^{T}
\begin{bmatrix}
\0 & \frac{1}{2}\I \\
\frac{1}{2}\I & \0
\end{bmatrix}
\begin{bmatrix}
\Delta \dot{\q} \\
\Delta \tau
\end{bmatrix}.
\end{equation}
\section{Evaluation results} \label{sec:evaluation}
The proposed incrementally dissipative network (DMLP) is used to identify the discrete-time input--output dynamics of the considered systems. 
To isolate the effect of enforcing incremental dissipativity by construction, we compare DMLP against a baseline multilayer perceptron (MLP) trained purely by supervised learning without any dissipativity constraint.

The incremental dissipativity of the proposed DMLP is evaluated using the discrete-time incremental inequality
\begin{equation}\label{eq:inc_diss_eval}
    \Delta y_k \Delta u_k - \Delta V_k \geq 0,
\end{equation}
where $\Delta u_k := u_k-u_{k-1}$, $\Delta y_k := y_k-y_{k-1}$, and $\Delta V_k := V_k - V_{k-1}$ is the storage increment. 
The inequality in \eqref{eq:inc_diss_eval} is checked along the rollout trajectories and visualized in the dissipativity plots. Implementation configurations and hyperparameters are showen in Appendix~\ref{appx:hyperparameter}.
\subsection{Results on the prediction performance of MSD and n-DOF manipulator dynamics with and without training noises}
The evaluation results for the MSD system are shown in Fig.~\ref{fig:compare_msd_all}. 
Fig.~\ref{fig:compare_msd_incremental} compares the baseline MLP and the proposed DMLP under external forcing inputs. 
Although both models capture the main oscillatory behavior of position and velocity, the MLP shows larger transient deviations, whereas DMLP provides more consistent tracking. 
Moreover, the dissipativity plot verifies that DMLP satisfies the incremental dissipativity condition along the rollout. 
This is evaluated using the training input (Input 1) and an additional disturbed input (Input 2), as shown in the top subplot. 
In contrast, Fig.~\ref{fig:mlp_diss_check} shows that the standard MLP fails to preserve incremental dissipativity, evaluated using Eq.~(3) with known MSD parameters.

Fig.~\ref{fig:compare_msd_incremental_noise} presents the noisy-training case, where the outputs are corrupted by Gaussian noise. 
The baseline MLP becomes more sensitive to noise and exhibits larger prediction errors, while DMLP maintains accurate predictions and preserves incremental dissipativity. 
The quantitative results in Table~\ref{tab:comparison_all_systems} further confirm this trend. 
For example, when $N=100$ under noisy training, DMLP reduces the mean prediction error from $38.9\times10^{-2}$ to $14.3\times10^{-2}$, demonstrating its robustness and advantage in the low-data regime.

\begin{table}[t]
\centering
\caption{Mean prediction errors ($\pm$ std) of MLP and DMLP on the MSD, 2-DOF, and 3-DOF systems for different training sizes $N$ and noise settings.}
\label{tab:comparison_all_systems}
\renewcommand{\arraystretch}{1.25}
\setlength{\tabcolsep}{8pt}
\begin{tabular}{c|c|c|c|c}
\hline
\textbf{System} & \textbf{$N$} & \textbf{Noise} & \textbf{MLP} & \textbf{DMLP} \\
\hline

\multirow{4}{*}{MSD}
& \multirow{2}{*}{100}
& Yes & $38.9 \pm 0.02$ & $14.3 \pm 0.02$ \\
\cline{3-5}
& & No  & $11.1 \pm 0.03$ & $8.21 \pm 0.01$ \\
\cline{2-5}
& \multirow{2}{*}{200}
& Yes & $5.02 \pm 0.01$ & $4.42 \pm 0.00$ \\
\cline{3-5}
& & No  & $2.53 \pm 0.02$ & $3.06 \pm 0.01$ \\
\hline

\multirow{4}{*}{2-DOF}
& \multirow{2}{*}{100}
& Yes & $32.80 \pm 0.02$ & $15.3 \pm 0.01$ \\
\cline{3-5}
& & No  & $23.2 \pm 0.01$  & $14.7 \pm 0.02$ \\
\cline{2-5}
& \multirow{2}{*}{200}
& Yes & $10.3 \pm 0.03$ & $11.2 \pm 0.01$ \\
\cline{3-5}
& & No  & $6.96 \pm 0.02$ & $11 \pm 0.03$ \\
\hline

\multirow{4}{*}{3-DOF}
& \multirow{2}{*}{100}
& Yes & $20.0 \pm 0.02$ & $6.9 \pm 0.01$ \\
\cline{3-5}
& & No  & $10.3 \pm 0.01$ & $5.84 \pm 0.02$ \\
\cline{2-5}
& \multirow{2}{*}{200}
& Yes & $5.3 \pm 0.02$ & $5.95 \pm 0.03$ \\
\cline{3-5}
& & No  & $2.44 \pm 0.01$ & $5.79 \pm 0.02$ \\
\hline

\end{tabular}
\vskip -0.1in
\end{table}

Similar improvements are observed for the manipulator systems: for the 2-DOF system with $N=100$, DMLP reduces the error from $32.80\times10^{-2}$ to $15.3\times10^{-2}$ under noisy training and from $23.2\times10^{-2}$ to $14.7\times10^{-2}$ without noise. 
For the 3-DOF system, the improvement is even more pronounced, with the error reduced from $20.0\times10^{-2}$ to $6.9\times10^{-2}$ under noisy training and from $10.3\times10^{-2}$ to $5.84\times10^{-2}$ without noise. 
These results suggest that the proposed DMLP provides improved robustness and data efficiency, particularly when the amount of training data is limited. The superior modeling performance of DMLP over the standard MLP for the 2-DOF and 3-DOF manipulator dynamics is also visually confirmed in Fig.~\ref{fig:compare_ndof_all}, Appendix~\ref{appendix:ndof_manipulator}.
\subsection{Training loss update}
In addition to trajectory prediction, we compare the training loss histories of the baseline model and the dissipativity-constrained models.
Fig.~\ref{fig:compare_all_loss} reports the loss curves for the baseline MLP, the original dissipative network (DMLP), and the improved dissipative network (DMLP, improved) in the noise-free setting for MSD, 2-DOF manipulator and 3-DOF manipulator. 
The baseline MLP typically converges quickly at the beginning of training, but may exhibit less stable optimization behavior.
In contrast, the dissipativity-constrained training yields smoother and more consistent loss decay. 
Moreover, the improved dissipative training procedure achieves the lowest and most stable loss profile, confirming the effectiveness of the proposed algorithmic refinement.
\begin{figure}[t]
\centering
\includegraphics[width=0.95\linewidth]{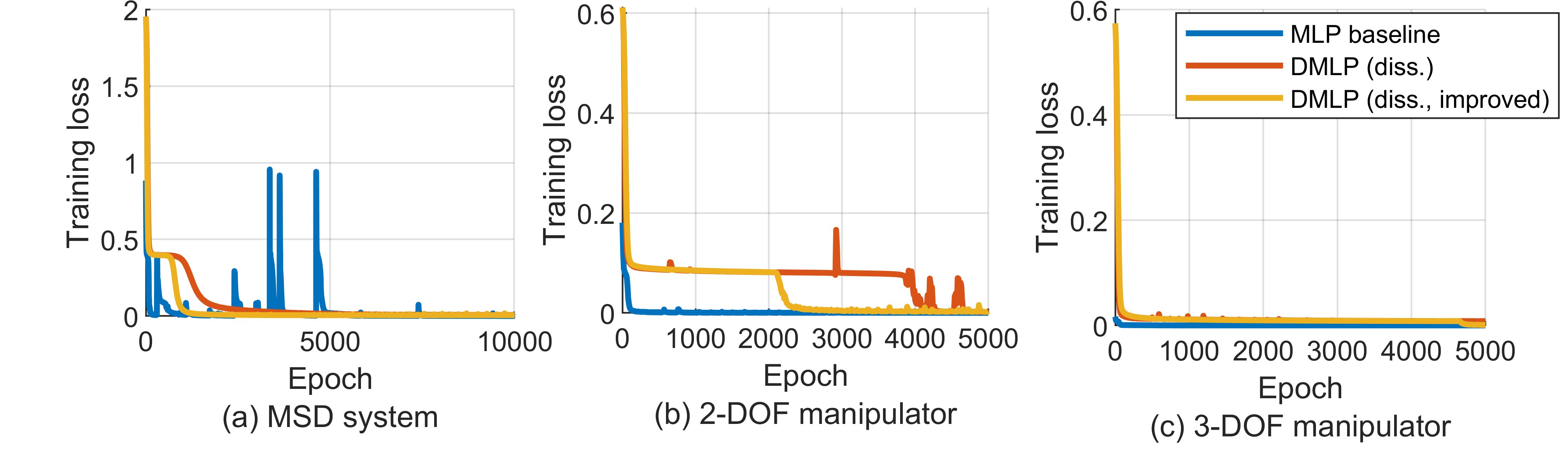}
\caption{Training loss histories of the baseline model, the original dissipative network, and the improved dissipative network, demonstrating the effectiveness of the improved training algorithm. 
(a) MSD system. (b) 2-DOF manipulator. (c) 3-DOF manipulator.}
\label{fig:compare_all_loss}
\end{figure}
\subsection{Comparison with baseline methods}
In addition to demonstrating the modeling performance of our proposed method compared with the standard MLP baseline, we have added a comparison with the state-of-the-art work of~\cite{okamoto2025learning}, which models dissipative systems using a continuous-time Neural ODE approach. We utilized their open-source implementation and evaluated it on the three dynamic systems. As summarized in Table~\ref{tab:performance_comparison}, the proposed DMLP achieves competitive prediction accuracy while being significantly more computationally efficient. For example, on the MSD system, DMLP requires $491.7 \pm 5$ seconds, whereas the Neural ODE approach requires $6.33 \pm 0.2$ hours. These results demonstrate that our discrete-time DMLP provides a practical and computationally efficient alternative for learning dissipative dynamics with comparable modeling performance.
\begin{table}[H]
\centering
\caption{Performance comparison between MLP, DMLP, and Dissipative Neural ODE.}
\label{tab:performance_comparison}
\resizebox{\linewidth}{!}{%
\renewcommand{\arraystretch}{1.35}
\setlength{\tabcolsep}{6pt}
\begin{tabular}{c|c|c|c|c}
\hline
\textbf{Model} 
& \textbf{Performance} 
& \textbf{MLP} 
& \textbf{DMLP} 
& \makecell{\bfseries Dissipative\\Neural ODE} \\
\hline

\multirow{2}{*}{\makecell{MSD\\(10000 epochs)}}
& Mean error
& $2.53 \pm 0.02 \;(\times 10^{-2})$
& $3.06 \pm 0.01 \;(\times 10^{-2})$
& $0.108 \pm 0.03$ \\
\cline{2-5}
& Computation time
& $403 \pm 10$ s
& $491.7 \pm 5$ s
& $6.33 \pm 0.2$ h \\
\hline

\multirow{2}{*}{\makecell{2-DOF manipulator\\(5000 epochs)}}
& Mean error
& $6.96 \pm 0.02 \;(\times 10^{-2})$
& $11 \pm 0.03 \;(\times 10^{-2})$
& $0.15 \pm 0.02$ \\
\cline{2-5}
& Computation time
& $262 \pm 11$ s
& $442 \pm 8$ s
& $4.15 \pm 0.5$ h \\
\hline

\multirow{2}{*}{\makecell{3-DOF manipulator\\(5000 epochs)}}
& Mean error
& $2.44 \pm 0.01 \;(\times 10^{-2})$
& $5.79 \pm 0.02 \;(\times 10^{-2})$
& $0.10 \pm 0.01$ \\
\cline{2-5}
& Computation time
& $302 \pm 15$s
& $525 \pm 10$s
& $7.22 \pm 0.7$ hour \\
\hline
\end{tabular}%
}
\end{table}
An ablation study that confirmed again the effectiveness of our method can be seen in Appendix~\ref{appx:ablation}.
\section{Conclusions} \label{sec:conclusion}
In this work, we proposed a dissipative network model that can learn the dynamics and preserve the dissipative property of a known dissipative nonlinear model. This was realized through the design and learning algorithm of the network structure so that the weights satisfy dissipative constraints and they can be solved with non-constrained optimization method. The proposed network was evaluated and compared with a naive model on a mass-spring-damper system and a flexible joint model, which are incrementally dissipative with particular input-output pairs. The results showed the effectiveness performance of our dissipative network compared in guaranteeing both good modeling performance and physical meaning preservation. 

This is expected to provide a useful tool for designing a more effective and physically meaningful network applicable across a wide range of applications. Research on the modeling of the proposed model on more complex case studies, and extension of the idea to other types of networks will be our future work.

\appendix

\section{\protect\proofdeltaVtitle} 
\label{sec:provedeltaV}
Left and right multiplying Eq.~(\ref{eq:lmi}) by $[\Delta \x_k^T \quad \Delta \z_{k}^T \quad \Delta \u_k^T]$ and $[\Delta \x_k^T \quad \Delta \z_{k}^T \quad \Delta \u_k^T]^T$ respectively, one can obtain
\begin{multline} \label{eq:qsr_expand}
    \Delta \x_k^T \P \Delta \x_k - \Delta \z_{k}^T \B^T \P \B \Delta \z_{k} \\ 
    + \Delta  \x_k^T \C^T \Q \C \Delta \x_k + 2\Delta \x_k^T \D^T \Q \C \Delta \u_k + 2\Delta \x_k^T \C^T \S \Delta \u_k \\ 
    + \Delta \u_k^T \R_1\Delta \u_k  - 2 \gamma \Delta \v_k^T \LambdaB \Delta \z_{k}   + \Delta \z_{k}^T \LambdaB \Delta \z_{k} \geq 0 
\end{multline} 
taking into account that $\Delta \v_k = \W_{u} \Delta \u_k + \W_{2} \Delta \z_{k} + \W_1\Delta \x_{k}$, and $V_{k+1}(\Delta \x_{k+1})-V_{k}(\Delta \x_{k}) =\Delta \x_{k}^T \B^T\P\B \Delta \x_{k} - \Delta \x_{k}^T \P \Delta \x_{k}$, and 
\begin{multline} \label{eq:yu_expand}
    [\Delta \y_k^T \quad \Delta \u_k^T]
    \begin{bmatrix}
        \Q & \S^T \\
        \S    & \R
    \end{bmatrix}
    \begin{bmatrix}
        \Delta \y_k\\
        \Delta \u_k 
    \end{bmatrix} \\
= \Delta \x_k^T \C^T \Q \C \Delta \x_k + 2\Delta \x_k^T \D^T \Q \C \Delta \u_k + \\
2\Delta \x_k^T \C^T \S \Delta \u_k + \Delta \u_k^T (\R +\D^T \Q \D)\Delta \u_k 
\end{multline}
Eq.~(\ref{eq:qsr_expand}) then can be re-written as
\begin{multline}
    V_{k+1}(\Delta \x_{k+1})-V_{k}(\Delta \x_{k}) \leq 
     [\Delta \y_k^T \quad \Delta \u_k^T]
    \begin{bmatrix}
        \Q & \S^T \\
        \S    & \R
    \end{bmatrix}
    \begin{bmatrix}
        \Delta \y_k\\
        \Delta \u_k 
    \end{bmatrix}\\
    + [\Delta \v_k^T \quad \Delta \z_{k}^T]
    \begin{bmatrix}
        \rho \LambdaB & -\gamma \LambdaB \\
         -\gamma \LambdaB    & \LambdaB
    \end{bmatrix}
    \begin{bmatrix}
        \Delta \v_k\\
        \Delta \z_{k} 
    \end{bmatrix}
\end{multline}

\section{\protect\proofSlopeTitle}
\label{sec:slopeproof}
The slopes of widely-used activation functions $\phi()$, such as ReLU, tanh, exponential linear functions, are restricted by the range [alpha, beta]. We  have
\begin{equation}
    \alpha \Delta \v_k \leq \Delta \z_{k} = \phi(\v^{(1)}_k)-\phi(\v^{(2)}_k)\leq \beta \Delta \v_k
\end{equation}
Therefore
\begin{equation}
    (\Delta \z_{k}-\alpha \Delta \v_k)\LambdaB ( \Delta \z_{k} -\beta \Delta \v_k)
\end{equation}

\begin{equation} \label{eq:slope}
    [\Delta \v_k^T \quad \Delta \z_{k}^T]
    \begin{bmatrix}
        \rho \LambdaB & -\gamma \LambdaB \\
         -\gamma \LambdaB    & \LambdaB
    \end{bmatrix}
    \begin{bmatrix}
        \Delta \v_k\\
        \Delta \z_{k}
    \end{bmatrix} \leq 0
\end{equation}
where, $\rho = \alpha \beta, \gamma = \frac{1}{2}(\alpha + \beta)$. 
\section{Choose $\D$ such that} \label{sec:D}
\begin{equation} \label{eq:R}
    \R_1 = \R + \S \D + \D^T \S^T + \D^T \Q \D \succ 0
\end{equation}
It is noticed that Eq.~(\ref{eq:R}) will be satisfied if
\begin{equation} \label{eq:D_cond}
\L = \R + \S \D + \D^T \S^T + \D^T \Q_1 \D \succ 0 
\end{equation}
where $\Q_1 = \Q - \epsilon \I $ with $\I \in \mathbb{R}^{p \times p}$, and $\epsilon$ is a sufficiently small constant to guarantee that $\Q_1 \prec 0$. $\epsilon$ can be chosen as 0 in case $\Q \prec \0$. Then $\Q_1$ can be represented as $\Q_1 = -\L_q^T \L_q$.

$\D$ can be chosen as follow~\cite{revay2023recurrent}
\begin{equation} \label{eq:D_cond3}
    \D = -\Q_1^{-1} \S^T + \L_q^{-1} \N \L_r
\end{equation}
where $\L_r\in \mathbb{R}^{m \times m}$ is found from Eq.~(\ref{eq:Lr})
\begin{equation} \label{eq:Lr}
    \R -\S \Q_1 \S^{T} = \L_r^T \L_r
\end{equation}
$\N$ and $\M$ can be chosen as follows

if $p \geq m$
\begin{align}
\label{eq:MN}
    \M=\Y_1^T\Y_1 + \Y_2 - \Y_2^T + \Y_3^T\Y_3 + \epsilon_3 \I_m \\
    \N = \begin{bmatrix}
        (\I_m-\M)(\I_m+\M)^{-1}\\
        -2\Y_3(\I_m + \M)^{-1}
    \end{bmatrix}
\end{align} 
where $\Y_1\in \mathbb{R}^{m \times m}, \Y_2\in \mathbb{R}^{m \times m}$, and $\Y_3\in \mathbb{R}^{(p-m) \times m}$ are free variables, $\epsilon_3 > 0$, $\I_m$ is the identity matrix with size $ m \times m$.

if $p < m$
\begin{align}
\label{eq:MN2}
    \M=\Y_1^T\Y_1 + \Y_2 - \Y_2^T + \Y_3^T\Y_3 + \epsilon_3 \I_p \\
    \N = \begin{bmatrix}
        (\I_m+\M)^{-1}(\I_m-\M)  \quad      -2(\I_m + \M)^{-1}\Y_3^T
    \end{bmatrix}
\end{align} 
where $\Y_1\in \mathbb{R}^{p \times p}, \Y_2\in \mathbb{R}^{p \times p}$, and $\Y_3\in \mathbb{R}^{(m-p) \times p}$ are free variables, $\I_p$ is the identity matrix with size $ p \times p$.

It can be proven that $\N\N^T \prec \I$ and $\R_1 = \L_r^T(\I - \N\N^T)\L_r \succ 0$. where $\I$ the identity matrix with the same size as $\N\N^T$.
\section{Incremental dissipativity of MSD.} \label{appendix:msd}
Consider
\[
m\ddot x+b\dot x+kx=u,\qquad y=\dot x,
\]
with $m>0,\;b\ge0,\;k>0$. For two trajectories $(x_i,\dot x_i,u_i)$, define
\[
\delta x=x_1-x_2,\quad \delta v=\dot x_1-\dot x_2,
\]
\[
\quad \delta u=u_1-u_2,\quad \delta y=y_1-y_2=\delta v.
\]
Subtracting dynamics gives
\[
m\delta\dot v+b\delta v+k\delta x=\delta u,\qquad \delta\dot x=\delta v.
\]
Choose storage
\[
V=\frac12 k(\delta x)^2+\frac12 m(\delta v)^2.
\]
Then
\[
\dot V=k\delta x\,\delta\dot x+m\delta v\,\delta\dot v
      =k\delta x\,\delta v+\delta v(\delta u-b\delta v-k\delta x)
\]
\[
      =\delta y\,\delta u-b(\delta y)^2
      \le \delta y\,\delta u.
\]
Hence the MSD system is incrementally dissipative (incrementally passive) from $u$ to $y=\dot x$.
\section{Incremental dissipativity of the $n$-DOF manipulator and modeling performance}
\label{appendix:ndof_manipulator}

Consider the $n$-DOF manipulator dynamics
\begin{equation}
\M\ddot{\q} + \D\dot{\q} + \K \q = \tau,
\label{eq:ndof_manipulator}
\end{equation}
where \(\q,\tau\in\mathbb{R}^n\). Assume that
\[
\M=\M^\top\succ0,\qquad
\D=\D^\top\succeq0,\qquad
\K=\K^\top\succeq0 .
\]
Define the input--output pair
\begin{equation}
\u=\t,\qquad \y=\dot{\q}.
\end{equation}

Let \((\q_1,\tau_1)\) and \((\q_2,\tau_2)\) be two trajectories of
Eq.~\eqref{eq:ndof_manipulator}. Define the incremental variables
\[
\Delta \q := \q_1-\q_2,\qquad
\Delta \dot{\q} := \dot{\q}_1-\dot{\q}_2,\qquad
\Delta \tau := \tau_1-\tau_2 .
\]
Subtracting the two systems gives
\begin{equation}
\M\Delta \ddot{\q} + \D\Delta \dot{\q} + \K\Delta \q = \Delta \tau .
\label{eq:inc_ndof_manipulator}
\end{equation}

Choose the incremental storage function
\begin{equation}
V(\Delta)
:=
\frac12 \Delta \dot{\q}^{\top}\M\Delta \dot{\q}
+
\frac12 \Delta \q^{\top}\K\Delta \q .
\label{eq:storage_ndof_manipulator}
\end{equation}
Since \(\M\succ0\) and \(\K\succeq0\), we have \(V(\Delta)\ge0\).

Taking the time derivative of \(V\), we obtain
\begin{align}
\dot V
&=
\Delta \dot{\q}^{\top}\M\Delta \ddot{\q}
+
\Delta \q^{\top}\K\Delta \dot{\q}.
\label{eq:Vdot_ndof_1}
\end{align}
Using Eq.~\eqref{eq:inc_ndof_manipulator}, we have
\[
\M\Delta \ddot{\q}
=
\Delta \tau
-
\D\Delta \dot{\q}
-
\K\Delta \q .
\]
Substituting this into Eq.~\eqref{eq:Vdot_ndof_1} gives
\begin{align}
\dot V
&=
\Delta \dot{\q}^{\top}
\left(
\Delta \tau
-
\D\Delta \dot{\q}
-
\K\Delta \q
\right)
+
\Delta \q^{\top}\K\Delta \dot{\q}
\nonumber\\
&=
\Delta \dot{\q}^{\top}\Delta \t
-
\Delta \dot{\q}^{\top}\D\Delta \dot{\q}
-
\Delta \dot{\q}^{\top}\K\Delta \q
+
\Delta \q^{\top}\K\Delta \dot{\q}.
\end{align}
Because \(\K=\K^\top\), we have
\[
-\Delta \dot{\q}^{\top}\K\Delta \q
+
\Delta \q^{\top}\K\Delta \dot{\q}
=0.
\]
Therefore,
\begin{equation}
\dot V
=
\Delta \dot{\q}^{\top}\Delta \tau
-
\Delta \dot{\q}^{\top}\D\Delta \dot{\q}.
\end{equation}
Since \(\D\succeq0\), it follows that
\begin{equation}
\dot V
\le
\Delta \dot{\q}^{\top}\Delta \tau .
\label{eq:inc_dissipativity_ndof_manipulator}
\end{equation}

Thus, the $n$-DOF manipulator is incrementally dissipative from
\[
\u=\tau
\]
to
\[
\y=\dot{\q}
\]
with supply rate
\begin{equation}
w(\Delta \u,\Delta \y)
=
\Delta \u^{\top}\Delta \y
=
\Delta \tau^{\top}\Delta \dot{\q},
\end{equation}
and incremental storage function given by Eq.~\eqref{eq:storage_ndof_manipulator}.
\begin{figure}[H]
\centering

\subfloat[2-DOF manipulator]{
    \includegraphics[width=0.45\linewidth]{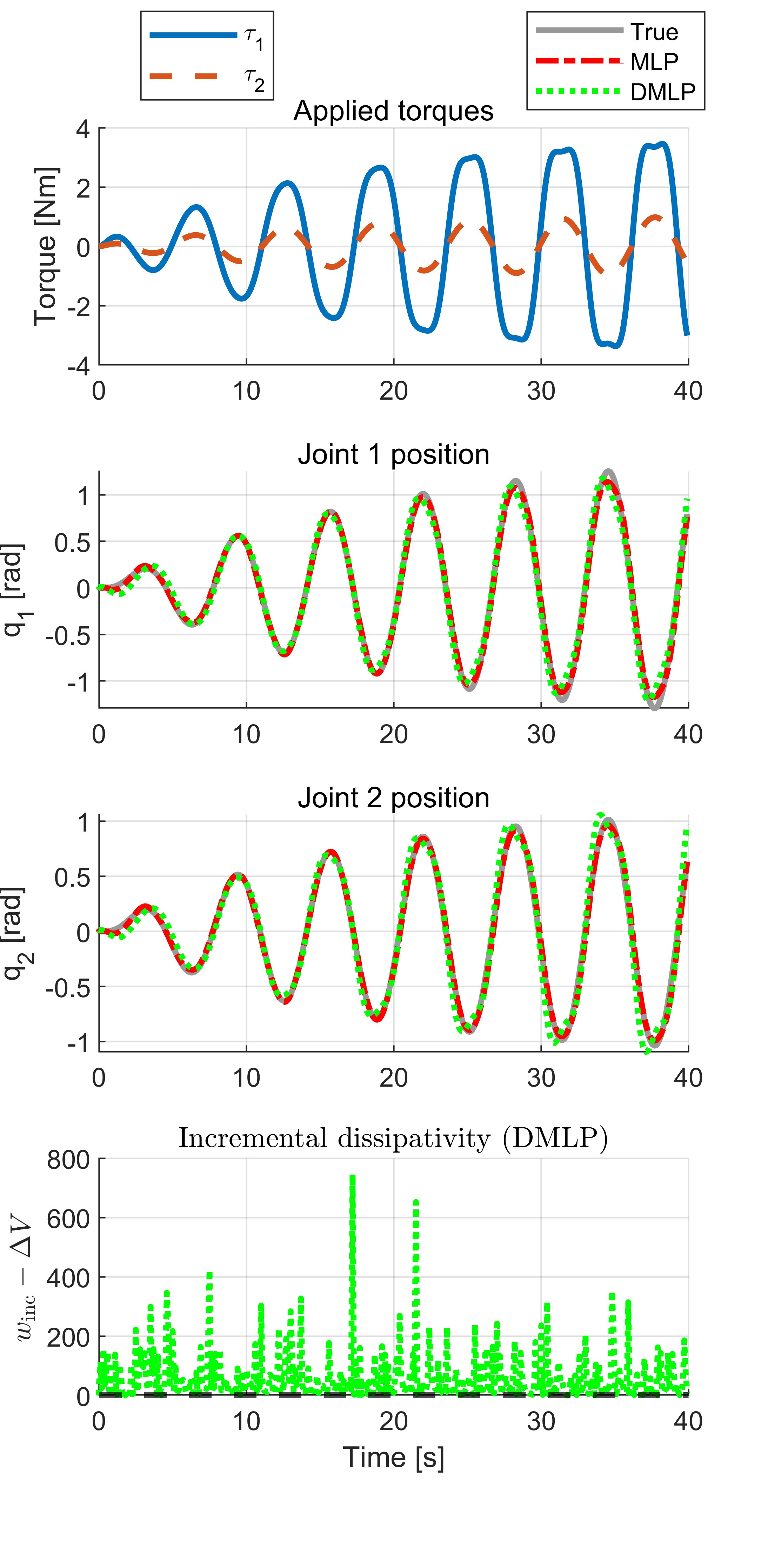}
    \label{fig:compare_2dof_incremental}
}
\hfill
\subfloat[3-DOF manipulator]{
    \includegraphics[width=0.45\linewidth]{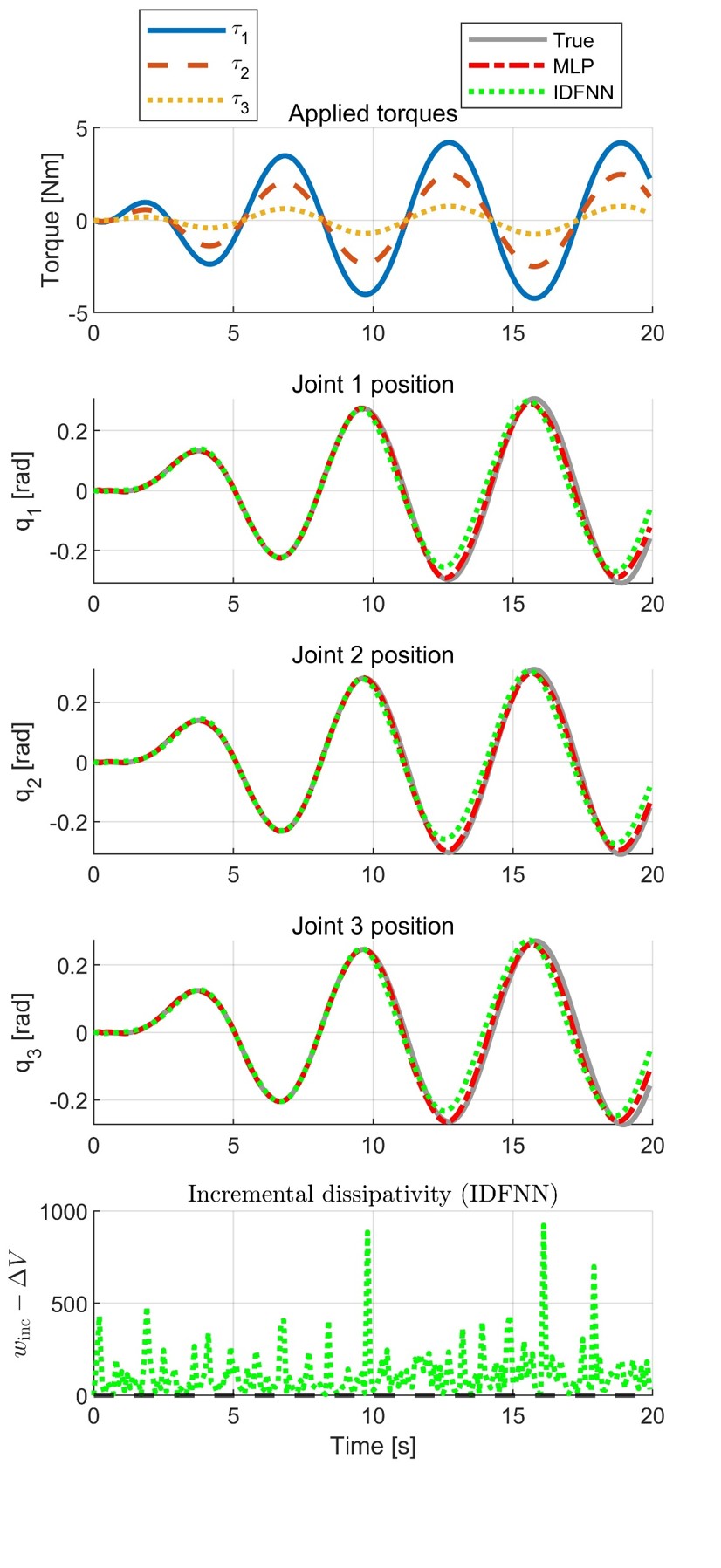}
    \label{fig:3dof_compare_plot}
}

\caption{Comparison of the performance between baseline MLP model and the dissipative model. 
(a) 2-DOF manipulator without training noise, where data from $0\,\mathrm{s}$ to $20\,\mathrm{s}$ were used for training and data from $20\,\mathrm{s}$ to $40\,\mathrm{s}$ were used for testing. (b) 3-DOF manipulator without training noise, where data from $0\,\mathrm{s}$ to $10\,\mathrm{s}$ were used for training and data from $10\,\mathrm{s}$ to $20\,\mathrm{s}$ were used for testing.  In both cases, besides the input used in training, a disturbed input signal was also used to check the incremental dissipativity condition of the proposed network.}
\label{fig:compare_ndof_all}
\end{figure}
\section{Ablation studies} \label{appx:ablation}
An additional ablation study was conducted to further evaluate the contribution of each component of the proposed method. 
As summarized in Table~\ref{tab:ablation_study}, we compare the unconstrained MLP baseline with the dissipativity-constrained DMLP, evaluate the proposed model with and without the improved training algorithm, and test different network configurations, including different depths and activation functions. 
These comparisons are performed across multiple systems, including the MSD system, the 2-DOF manipulator, and the 3-DOF manipulator.

The results show that the dissipativity-constrained DMLP provides more reliable modeling performance than the unconstrained MLP baseline. 
Moreover, the improved training algorithm consistently enhances the performance of DMLP while reducing the training time across the considered systems. 
The additional comparisons with different network depths and activation functions further demonstrate that the proposed framework can be applied under various architectural choices. 
Overall, the ablation results confirm the effectiveness and scalability of the proposed method across systems with increasing dynamical complexity.
\begin{table}[t]
\centering
\caption{Ablation study with the number of data samples $N=100$.}
\label{tab:ablation_study}

\renewcommand{\arraystretch}{1.35}
\setlength{\tabcolsep}{4pt}

\resizebox{\linewidth}{!}{%
\begin{tabular}{cc|cc|cc|cc}
\hline
\multicolumn{2}{c|}{\multirow{2}{*}{\textbf{Ablation cases}}}
&
\multicolumn{2}{c|}{\makecell{\bfseries MSD\\\bfseries (10000 epochs)}}
&
\multicolumn{2}{c|}{\makecell{\bfseries 2-DOF manipulator\\\bfseries (5000 epochs)}}
&
\multicolumn{2}{c}{\makecell{\bfseries 3-DOF manipulator\\\bfseries (5000 epochs)}}
\\
\cline{3-8}

\multicolumn{2}{c|}{}
&
\makecell{Mean error\\$(\times 10^{-2})$}
&
\makecell{Training\\time}
&
\makecell{Mean error\\$(\times 10^{-2})$}
&
\makecell{Training\\time}
&
\makecell{Mean error\\$(\times 10^{-2})$}
&
\makecell{Training\\time}
\\
\hline

\multirow{2}{*}{\makecell{Dissipativity\\constraints}}
& No (MLP)
& $2.53 \pm 0.02$
& $403 \pm 10$ s
& $6.96 \pm 0.02$
& $262 \pm 11$ s
& $2.44 \pm 0.01$
& $302 \pm 15$ s
\\

& Yes (DMLP)
& $3.06 \pm 0.01$
& $491.7 \pm 5$ s
& $11 \pm 0.03$
& $423 \pm 10$ s
& $5.79 \pm 0.02$
& $525 \pm 10$ s
\\
\hline

\multirow{2}{*}{\makecell{Training algorithm\\(DMLP)}}
& Improved
& $3.06 \pm 0.01$
& $491.7 \pm 13$ s
& $11 \pm 0.03$
& $423 \pm 10$ s
& $5.79 \pm 0.02$
& $525 \pm 10$ s
\\

& Original
& $4.98 \pm 0.02$
& $625 \pm 12$ s
& $13.4 \pm 0.02$
& $472 \pm 12$ s
& $6.02 \pm 0.03$
& $578 \pm 11$ s
\\
\hline

\multirow{4}{*}{\makecell{Number of\\layers\\(DMLP)}}
& 1
& $5.02 \pm 0.01$
& $500 \pm 15$ s
& $15.5 \pm 0.01$
& $412 \pm 13$ s
& $7.33 \pm 0.02$
& $502 \pm 13$ s
\\

& 2
& $3.06 \pm 0.01$
& $491.7 \pm 5$ s
& $11 \pm 0.03$
& $423 \pm 10$ s
& $5.79 \pm 0.02$
& $525 \pm 10$ s
\\

& 3
& $4.89 \pm 0.03$
& $702 \pm 11$ s
& $13.8 \pm 0.03$
& $445 \pm 15$ s
& $6.02 \pm 0.03$
& $565 \pm 9$ s
\\

& 4
& $4.92 \pm 0.04$
& $732 \pm 10$ s
& $12.3 \pm 0.04$
& $492 \pm 13$ s
& $6.25 \pm 0.06$
& $602 \pm 15$ s
\\
\hline

\multirow{3}{*}{\makecell{Activation\\functions\\(DMLP)}}
& Tanh
& $3.06 \pm 0.01$
& $491.7 \pm 5$ s
& $11 \pm 0.03$
& $423 \pm 10$ s
& $5.79 \pm 0.02$
& $525 \pm 10$ s
\\

& Sigmoid
& $4.02 \pm 0.03$
& $633 \pm 11$ s
& $13.6 \pm 0.01$
& $481 \pm 11$ s
& $6.36 \pm 0.05$
& $536 \pm 10$ s
\\

& ReLU
& $4.08 \pm 0.01$
& $614 \pm 10$ s
& $14.0 \pm 0.03$
& $416 \pm 9$ s
& $5.72 \pm 0.06$
& $510 \pm 8$ s
\\
\hline

\end{tabular}%
}
\end{table}
\section{Implementation settings and Hyperparameters} \label{appx:hyperparameter}
All simulations were conducted on a desktop workstation with an Intel(R) Core(TM) i9-9900K CPU @ 3.6\,GHz, 128\,GB RAM, and an NVIDIA GeForce RTX 2080 Ti GPU. 
Both MLP and DMLP were trained by minimizing the mean squared error (MSE) loss between the predicted and true system outputs. 
Specifically, the loss function is defined using \texttt{nn.MSELoss()}, and the model parameters are optimized using the Adam optimizer with learning rate $10^{-3}$. 

In our experiments, approximately $25$--$50\%$ of the available dataset is used for training, while the remaining samples are reserved for testing. 
This split allows us to evaluate the generalization ability of the learned model under limited training data.
For the MSD system, all models were trained for 10000 epochs, while for the 2-DOF and 3-DOF manipulator systems all models were trained for 5000 epochs. The hyperparameters of the DMLP and MLP networks are identical and are in Table~\ref{tab:simulation_hyperparameters}.
\begin{table}[H]
\centering
\caption{Hyperparameters used in the simulation.}
\label{tab:simulation_hyperparameters}
\renewcommand{\arraystretch}{1.2}
\setlength{\tabcolsep}{8pt}
\begin{tabular}{lccc}
\hline
\textbf{Hyperparameter} & \textbf{MSD} & \textbf{2-DOF} & \textbf{3-DOF} \\
\hline
Input size & 1 & 2 & 3 \\
Hidden size & 32 & 32 & 32 \\
Output size & 2 & 4 & 6 \\
Number of layers & 2 & 2 & 2 \\
Activation function & tanh & tanh & tanh \\
\hline
\end{tabular}
\end{table}

\newpage


\bibliographystyle{ieeetr}
\bibliography{reference}

\end{document}